\documentclass[envcountsame]{llncs}

\newif\ifignore 
\ignorefalse
\newcommand{\auxproof}[1]{
\ifignore\mbox{}\newline
\textbf{PROOF:} \dotfill\newline
{\it #1}\mbox{}\newline
\textbf{ENDPROOF}\dotfill
\fi}

\usepackage{etex} 
\usepackage{amsmath}
\usepackage{amssymb}
\usepackage{amstext}
\usepackage{mathtools}
\usepackage{mathrsfs}
\usepackage{stmaryrd}
\usepackage{xspace}
\usepackage{enumitem}
\usepackage{cmll}
\usepackage{hyperref}
\usepackage{url}
\usepackage{nicefrac}
\usepackage{fancybox}
\usepackage{listings}
\usepackage{wrapfig}
\usepackage{xypic}
\usepackage[all,cmtip]{xy}
\xyoption{2cell}
\UseTwocells
\xyoption{tips}

\usepackage{tikz}
\usetikzlibrary{circuits.ee.IEC}
\usetikzlibrary{shapes,shapes.geometric,shapes.misc}

\renewcommand{\arraystretch}{1.3}
\newcommand{\QEDbox}{\square}
\newcommand{\QED}{\hspace*{\fill}$\QEDbox$}

\makeatletter
\newcommand{\mathoverlap}[2]{\mathpalette\mathoverlap@{{#1}{#2}}}
\newcommand{\mathoverlap@}[2]{\mathoverlap@@{#1}#2}
\newcommand{\mathoverlap@@}[3]{\ooalign{$\m@th#1#2$\crcr\hidewidth$\m@th#1#3$\hidewidth}}
\makeatother

\newcommand{\klafter}{\mathrel{\bullet}}

\DeclareSymbolFont{T1op}{T1}{cmr}{m}{n}
\SetSymbolFont{T1op}{bold}{T1}{cmr}{bx}{n}
\DeclareMathSymbol{\mathguilsinglleft}{\mathopen}{T1op}{'016}
\DeclareMathSymbol{\mathguilsinglright}{\mathclose}{T1op}{'017}

\newcommand{\idmap}[1][]{\ensuremath{\mathrm{id}_{#1}}}
\newcommand{\after}{\mathrel{\circ}}

\newcommand{\st}{\ensuremath{\mathrm{st}}}
\newcommand{\ust}{\ensuremath{\mathrm{ust}}}
\newcommand{\dst}{\ensuremath{\mathrm{dst}}}

\newcommand{\dis}{\ensuremath{\mathrm{dis}}}

\newcommand{\row}{\ensuremath{\mathrm{row}}}

\newcommand{\PDF}{\ensuremath{\mathrm{PDF}}}
\newcommand{\Dir}{\ensuremath{\mathrm{Dir}}}

\newcommand{\set}[2]{\{#1\;|\;#2\}}
\newcommand{\setin}[3]{\{#1\in#2\;|\;#3\}}
\newcommand{\supp}{\mathrm{supp}}

\newcommand{\tuple}[1]{\langle#1\rangle}

\newcommand{\ket}[1]{\ensuremath{|{\kern.1em}#1{\kern.1em}\rangle}}
\newcommand{\bigket}[1]{\ensuremath{\big|{\kern.1em}#1{\kern.1em}\big\rangle}}

\newcommand{\incr}[2]{#1\!\mathrel{\mathrm{+{\kern-.1em}+}}\!#2}

\newcommand{\powersetsymbol}{\mathcal{P}}

\newcommand{\distributionsymbol}{\mathcal{D}}

\newcommand{\multisetsymbol}{\mathcal{M}}

\newcommand{\Pow}{\powersetsymbol}

\newcommand{\Dst}{\distributionsymbol}

\newcommand{\aneDst}{\distributionsymbol_{\circledast}}

\newcommand{\Mlt}{\multisetsymbol}
\newcommand{\neMlt}{\multisetsymbol_{*}}

\newcommand{\sneMlt}{\multisetsymbol_{*_2}\!}
\newcommand{\aneMlt}{\multisetsymbol_{\circledast}}
\newcommand{\Giry}{\mathcal{G}}

\newcommand{\UF}{\ensuremath{\mathcal{U}{\kern-.75ex}\mathcal{F}}}

\newcommand{\Cat}[1]{\ensuremath{\mathbf{#1}}\xspace}

\newcommand{\Kl}{\mathcal{K}{\kern-.4ex}\ell}
\newcommand{\KlN}{\Kl_{\mathbb{N}}}
\newcommand{\EM}{\mathcal{E}{\kern-.4ex}\mathcal{M}}

\newcommand{\Sets}{\Cat{Sets}}
\newcommand{\Nat}{\Cat{Nat}}

\newcommand{\NNO}{\mathbb{N}}
\newcommand{\pNNO}{\mathbb{N}_{>0}}

\newcommand{\R}{\mathbb{R}}
\newcommand{\nnR}{\R_{\geq 0}}

\newcommand{\Ef}{\ensuremath{\mathcal{E}{\kern-.5ex}f}}
\newcommand{\intd}{{\kern.2em}\mathrm{d}{\kern.03em}}
\newcommand{\indic}[1]{\mathbf{1}_{#1}}

\newcommand{\OF}{\ensuremath{\mathcal{O}{\kern-.1em}\mathcal{F}}}
\newcommand{\Closed}{\ensuremath{\mathcal{C}{\kern-.05em}\ell}}

\newcommand{\margsign}{\mathsf{M}}

\newcommand{\fM}{\ensuremath{\margsign_{1}}}
\newcommand{\sM}{\ensuremath{\margsign_{2}}}

\newsavebox\sbpto
\savebox\sbpto{\begin{tikzpicture}[baseline=-2.5pt]
            \filldraw[fill=white,draw=white] circle (1.4pt);
            \filldraw[fill=white,draw=black,line width=0.2pt] circle
(1.2pt);
                \end{tikzpicture}}
\newcommand\pto{\mathrel{\ooalign{$\to$\cr
            \hfil\!$\usebox\sbpto$\hfil\cr}}}
            
\newcommand\kto[2]{#1 \pto #2}

\makeatother
 \newdir{ >}{{}*!/-7.5pt/@{>}}
 \newdir{|>}{!/4.5pt/@{|}*:(1,-.2)@^{>}*:(1,+.2)@_{>}}
 \newdir{ |>}{{}*!/-3pt/@{|}*!/-7.5pt/:(1,-.2)@^{>}*!/-7.5pt/:(1,+.2)@_{>}}

\newsavebox\sbground
\savebox\sbground{\begin{tikzpicture}[circuit ee IEC,yscale=0.5,xscale=0.4]
                \draw (0,-2ex) to (0,0) node[ground,rotate=90,xshift=.65ex] {};
                \end{tikzpicture}}

\newcommand{\ie}{\textit{i.e.}\xspace}

\DeclareFixedFont{\ttb}{T1}{txtt}{bx}{n}{11} 
\DeclareFixedFont{\ttm}{T1}{txtt}{m}{n}{11}  
\usepackage{color}
\definecolor{deepblue}{rgb}{0,0,0.5}
\definecolor{deepred}{rgb}{0.6,0,0}
\definecolor{deepgreen}{rgb}{0,0.5,0}
\definecolor{lightgray}{rgb}{0.83,0.83,0.83}

\newcommand\pythonstyle{\lstset{
backgroundcolor = \color{lightgray},
language=Python,
basicstyle=\ttm,
otherkeywords={self,>>>,...},             
keywordstyle=\ttb\color{deepblue},
emph={MyClass,__init__},          
emphstyle=\ttb\color{deepred},    
stringstyle=\color{deepgreen},
frame=tb,                         
showstringspaces=false            %
}}

\lstnewenvironment{python}[1][]
{
\pythonstyle
\lstset{#1}
}
{}

\newcommand\pythoninline[1]{{\pythonstyle\lstinline!#1!}}

\begin{document}
\begin{frontmatter}
\title{Categorical Aspects of Parameter Learning}


\author{Bart Jacobs}

\institute{Institute for Computing and Information Sciences,
\\ 
Radboud University, Nijmegen, The Netherlands
\\ 
\email{bart@cs.ru.nl}
\\[+.5em]
\today
}

\maketitle

\begin{abstract}
Parameter learning is the technique for obtaining the probabilistic
parameters in conditional probability tables in Bayesian networks from
tables with (observed) data --- where it is assumed that the
underlying graphical structure is known. There are basically two ways
of doing so, referred to as maximal likelihood estimation (MLE) and as
Bayesian learning. This paper provides a categorical analysis of these
two techniques and describes them in terms of basic properties of the
multiset monad $\Mlt$, the distribution monad $\Dst$ and the Giry
monad $\Giry$. In essence, learning is about the reltionships between
multisets (used for counting) on the one hand and probability
distributions on the other. These relationsips will be described as
suitable natural transformations.
\end{abstract}
\end{frontmatter}


\renewcommand{\thepage}{\arabic{page}}

\section{Introduction}\label{IntroSec}

Bayesian networks are graphical models for efficiently organising
probabilistic
information~\cite{Barber12,BernardoS00,Darwiche09,JensenN07,KollerF09,Pearl88}. These
models can be used for probabilistic reasoning (inference), where the
probability of an observation is inferred from certain evidence. These
techniques are extremely useful, for instance in a medical setting,
where symptoms and measurements can be used as evidence, and the
inferred probability can help a doctor reach a decision.

A basic question is how to obtain accurate Bayesian networks. This
question involves two parts: how to determine the underlying graph
structure, and how to obtain the probabilities in the conditional
probability tables (CPTs) of the network. The first part is called
\emph{structure learning}, and the second part is called
\emph{parameter learning}. Here we concentrate on the latter,
especially for discrete probability distributions.

One way of obtaining the parameters of Bayesian network is to learn
them from experts. However, it is more efficient and cheaper to learn
the parameters from data, if available.  The data is typically
organised in (very large) tables, with information of the form: so
many patients had medicine $A$ and had these symptoms and so many
patients got medicine $B$ and showed such symptoms, etc.  Below we
shall describe such tables, say with $n$ dimensions, as $n$-ary
multisets in $\Mlt(X_{1}\times \cdots \times X_{n})$, where $\Mlt$ is
the multiset monad on the category of sets.

A multiset is a `set' in which elements may occur multiple times. We
shall write such multisets as formal combinations of the form
$3\ket{a} + 5\ket{b} + 2\ket{c}$. This expresses that we have $3$
occurrences of the element $a$, $5$ times $b$, and $2$ times $c$. Via
normalisation we can turn such a multiset into a probability
distribution, namely $\frac{3}{10}\ket{a} + \frac{5}{10}\ket{b} +
\frac{2}{10}\ket{c}$. Now the parameters add up to one, via division
by the sum of all occurrences. Below we shall describe this operation
as a natural transformation of the form $\neMlt \Rightarrow \Dst$,
where $\neMlt$ is the (sub)monad of \emph{non-empty} multisets, and
where $\Dst$ is the discrete probability distribution monad.  Various
properties of this natural transformation are identified, for instance
with respect to marginalisation and disintegration. These properties
are relevant for using this learning method with repect to a graph
structure.

The above learning technique extracts probability distributions
directly from the data, reformulated here via the transformation
$\neMlt \Rightarrow \Dst$. It is called \emph{maximal likelihood
  estimation} (MLE), see \textit{e.g.}~\cite[Ch.17]{Darwiche09},
\cite[\S17.1]{KollerF09} or~\cite[\S6.1.1]{JensenN07}. We like to see
it as `frequentist' learning, since it is based on counting and
frequencies. There is a more sophisticated form of learning called
\emph{Bayesian learning}, see~\cite[Ch.18]{Darwiche09},
\cite[\S17.3]{KollerF09} or~\cite[\S6.1.2]{JensenN07}. It is a form of
higher order learning, where one does not immediately obtain the
probability distribution in $\Dst(X)$, for a finite set $X$, but one
obtains a \emph{distribution over} $\Dst(X)$. The latter is a
continuous distribution, defined on a simplex, and thus involves the
Giry monad $\Giry$ on measurable spaces. Here we show how to
reformulate Bayesian learning in terms of another transformation, of
the form $\aneMlt \Rightarrow \Giry\Dst$, where $\aneMlt$ is the
(sub)monad of multisets in which each element occurs at least
once. The transformation is given by Dirichlet distributions. We show
how familiar properties of Dirichlet distributions translate into
categorical properties.

This paper thus gives a novel, snappy, categorical perspective on
parameter learning in terms of two natural transformations:
\[ \xymatrix{
\neMlt \ar@{=>}[rr]^{\text{frequentist}} & & \Dst
& \mbox{\quad and \quad} &
\aneMlt \ar@{=>}[rr]^{\text{Bayesian}} & & \Giry\Dst.
} \]

\noindent These transformations capture fundamental relationships
between the multiset monad $\Mlt$ on the one hand and probability
distribution monads $\Dst$ and $\Giry$ on the other hand. In the
practice of Bayesian networks, the differences between the frequentist
and Bayesian learning methods are substantial, for instance wrt.\ to
incorporating priors, variance, or zero counts, see
\textit{e.g.}~\cite[\S17.3]{KollerF09}. However, these differences are
not relevant in our abstract characterisations.

The topic of parameter learning is textbook material. The contribution
of this paper lies in a systematic, categorical reformulation that
offers a novel perspective --- not only in terms of natural
transformation as above, but also via conditioning, see
Section~\ref{LogicSec} --- that may lead to a deeper understanding and
to new connections. This reformulation involves a level of
mathematical precision that may be useful for (the semantics of)
probabilistic programming languages.




\section{Preliminaries on tables and distributions}\label{PrelimSec}

This section will elaborate a simple example in order to provide
background information about the setting in which parameter learning
is used.




Consider the table~\eqref{TableEqn} below where we have combined
numeric information about blood pressure (either high $H$ or low $L$)
and certain medicines (either type $1$ or type $2$ or no medicine,
indicated as $0$). There is data about 100 study participants:
\begin{equation}
\label{TableEqn}
\hbox{\begin{tabular}{r||c|c|c|c}
 & \textbf{ no medicine } & \textbf{ medicine 1 } & \textbf{ medicine 2 } 
    & \textbf{ totals }
\\
\hline\hline
\textbf{ high } & 10 & 35 & 25 & 70
\\
\textbf{ low } & 5 & 10 & 15 & 30
\\
\hline
\textbf{ totals } & 15 & 45 & 40 & 100
\end{tabular}}
\end{equation} 

\noindent We consider several ways to `learn' from this table.

\smallskip

\noindent \textbf{(1)} We can form the cartesian product $\{H,T\}
\times \{0,1,2\}$ of the possible outcomes and then capture the above
table as a multiset over this product:
\[ 10\ket{H,0} + 35\ket{H,1} + 25\ket{H,2} + 5\ket{L,0} + 10\ket{L,1} + 
  15\ket{L,2}. \]

\noindent We can normalise this multiset. It yields a joint
probability distribution, which we call $\omega$, for later
reference:
\[ \begin{array}{rcl}
\omega
& \coloneqq &
0.10\ket{H,0} \!+\! 0.35\ket{H,1} \!+\! 0.25\ket{H,2} \!+\!
   0.05\ket{L,0} \!+\! 0.10\ket{L,1} \!+\! 0.15\ket{L,2}
\end{array} \]

\noindent Such a distribution, directly derived from a table, is
sometimes called an \emph{empirical} distribution~\cite{Darwiche09}.

\smallskip

\noindent \textbf{(2)} The first and second marginals $\fM(\omega)$
and $\sM(\omega)$ of this joint probability distribution $\omega$
capture the blood pressure probabilities and the medicine
probabilities separately, as: $\fM(\omega) = 0.7\ket{H} + 0.3\ket{L}$
and $\sM(\omega) = 0.15\ket{0} + 0.45\ket{1} + 0.4\ket{2}$. These
marginal distributions can also be obtained directly from the above
table~\eqref{TableEqn}, via the normalisation of last column and last
row. This fact looks like a triviality, but involves a naturality
property (see Lemma~\ref{FrequentistNatLem} below).

\smallskip

\begin{wrapfigure}{r}{0pt}
$\vcenter{\xymatrix@R-1.5pc@C-1pc{
\ovalbox{\strut Blood}\ar[d] &
{\setlength\tabcolsep{0.3em}
   \renewcommand{\arraystretch}{1}
\begin{tabular}{|c|c|}
\hline
H & L \\
\hline\hline
$0.7$ & $0.3$ \\
\hline
\end{tabular}}
\\
\ovalbox{\strut Medicine}
&
{\setlength\tabcolsep{0.3em}
   \renewcommand{\arraystretch}{1}
\begin{tabular}{|c|c|c|c|}
\hline
& 0 & 1 & 2 \\
\hline\hline
$H$ & $\nicefrac{1}{7}$ & $\nicefrac{1}{2}$ & $\nicefrac{5}{14}$ \\
\hline
$L$ & $\nicefrac{1}{6}$ & $\nicefrac{1}{3}$ & $\nicefrac{1}{2}$ \\
\hline
\end{tabular}}
}}\quad
$
\end{wrapfigure}
\noindent \textbf{(3)} Next we wish to use the above
table~\eqref{TableEqn} to learn the parameters (table entries) for the
simple Bayesian network on the right. We then need to fill in the
associated conditional probability tables. These entries are obtained
from the last column in Table~\ref{TableEqn}, for the initial blood
distribution $0.7\ket{H} + 0.3\ket{L}$, and from the two rows in the
table; the latter yield two distributions for medicine usage, via
normalisation.

\smallskip

\noindent \textbf{(4)} In the categorical look at Bayesian networks
(see \textit{e.g.}~\cite{JacobsZ16,JacobsZ18}) these conditional
probability tables correspond to \emph{channels}: Kleisli maps for the
distribution monad $\Dst$. In the above case, the channel $c \colon
\kto{\{H,T\}}{\{0,1,2\}}$ corresponding to the medicine table
in the previous point is:
\[ \begin{array}{rclcrcl}
c(H)
& = &
\nicefrac{1}{7}\ket{0} + \nicefrac{1}{2}\ket{1} + \nicefrac{5}{14}\ket{2}
& \qquad &
c(L)
& = &
\nicefrac{1}{6}\ket{0} + \nicefrac{1}{3}\ket{1} + \nicefrac{1}{2}\ket{2}.
\end{array} \]

\noindent We can then recover the second marginal $\sM(\omega)$ as
state transformation $c \gg \fM(\omega)$, see later for details.

\smallskip

\noindent \textbf{(5)} Given a joint distribution $P(x,y)$ there is a
standard way to extract a channel $P(y\mid x)$ by taking conditional
probabilities.  This process is often called \emph{disintegration},
and is studied systematically in~\cite{ChoJ17a} (and in many other
places). If we disintegrate the above distribution $\omega$ on the
product $\{H,T\}\times\{0,1,2\}$ we obtain as channel
$\kto{\{H,T\}}{\{0,1,2\}}$ precisely the map $c$ from the previous
point --- obtained in point~\textbf{(3)} directly via the
Table~\eqref{TableEqn}. This is a highly relevant property, which
essentially means that (this kind of) learning can be done locally ---
see Proposition~\ref{FrequentistTableProp} below.

\smallskip

This example illustrates how probabilistic information can be
extracted from a table with numeric data --- in a frequentist manner
--- essentially by counting. This process will be analysed from a
systematic categorical perspective in Section~\ref{FrequentistSec}
below. The resulting structure will be useful in the subsequent more
advanced form of Bayesian learning, see Section~\ref{BayesianSec},
where continuous (Dirichlet) distributions are used on the
probabilistic parameters $r_{i}\in[0,1]$ in convex combinations
$\sum_{i}r_{i}\ket{i}$. But first we need to be more explicit about
the basic notions and notations that we use in our analysis.

\section{Prerequisites on multisets and discrete probability}\label{DiscPrereqSec}

Categorically, (finite) multisets can be captured via a the
\emph{multiset monad} $\Mlt$ on the category $\Sets$. For a set $X$
there is a new set $\Mlt(X) = \set{\phi\colon X \rightarrow
  \NNO}{\supp(\phi) \mbox{ is finite}}$ of multisets of $X$. The
support $\supp(\phi)$ is of a multiset $\phi$ is the subset
$\supp(\phi) = \setin{x}{X}{\phi(x) \neq 0}$ of its inhabitants. We
often write $\phi\in\Mlt(X)$ as formal finite sum $\phi =
\sum_{i}\alpha_{i}\ket{x_i}$, with support $\supp(\phi) = \{x_{1},
\ldots, x_{n}\}$ and $\phi(x_{i}) = \alpha_{i}\in\NNO$ telling how
often $x_{i}\in X$ occurs in the multiset $\phi$. The ket notation
$\ket{-}$ is meaningless syntactic sugar.


Each function $h\colon X \rightarrow Y$ gives rise to a function
$\Mlt(h) \colon \Mlt(X) \rightarrow \Mlt(Y)$ between the corresponding
collections of multisets. One defines $\Mlt(h)$ as:
\begin{equation}
\label{FunctorialEqn}
\hspace*{-0.5em}\begin{array}{rclcrcl}
\Mlt(h)(\phi)(y)
& = &
\!\!\displaystyle\sum_{x\in h^{-1}(y)} \phi(x)
& \quad\mbox{or as}\quad &
\Mlt(h)\big(\sum_{i}\alpha_{i}\ket{x_i}\big)
& = &
\sum_{i} \alpha_{i}\ket{h(x_{i})}.
\end{array}
\end{equation}

\noindent We do not need the monad structure of $\Mlt$ in this paper.
But functoriality --- the fact that $\Mlt$ not only acts on sets but
also on maps between them --- plays an important role. For instance,
the column and row of totals in Table~\eqref{TableEqn} are obtained as
`marginalisations' $\Mlt(\pi_{1})(\psi)\in\Mlt(\{H,T\})$ and
$\Mlt(\pi_{2})(\psi)\in\Mlt(\{0,1,2\})$ for the projection functions
$\{H,T\} \smash{\stackrel{\pi_1}{\longleftarrow}}
\{H,T\}\times\{0,1,2\} \smash{\stackrel{\pi_2}{\longrightarrow}}
\{0,1,2\}$.


Discrete probability distributions --- also called multinomials or
categorical distributions --- can be seen as special kinds of
multisets, not with natural numbers as multiplicities, but with
probabilities in $[0,1]$, with the additional requirement that these
probabilities add up to one (and thus form what is called a convex
combination). We write $\Dst(X)$ for the set of such discrete
probability distributions on $X$. It is defined as a set of
probability mass functions:
\begin{equation}
\label{DstEqn}
\begin{array}{rcl}
\Dst(X)
& \coloneqq &
\set{\varphi\colon X \rightarrow [0,1]}{\supp(\phi) \mbox{ is finite, and }
   \sum_{x}\phi(x) = 1}.
\end{array}
\end{equation}

\noindent An element of $\Dst(X)$ is often simply called a
distribution (or also a state) and is written as formal convex
combination $\sum_{i}r_{i}\ket{x_i}$, with $\sum_{i}r_{i} = 1$. On
$h\colon X\rightarrow Y$, a map $\Dst(h) \colon \Dst(X) \rightarrow
\Dst(Y)$ is defined essentially as in~\eqref{FunctorialEqn}. For a
projection $\pi_{1} \colon X\times Y \rightarrow X$ the associated
mapping $\Dst(\pi_{1}) \colon \Dst(X\times Y) \rightarrow \Dst(Y)$
performs marginalisation --- which we have written as $\fM$ in
point~\textbf{(2)} in Section~\ref{PrelimSec}. For more information
about the monads $\Mlt$ and $\Dst$ we refer to~\cite{Jacobs16g}.


We often identify a natural number $n$ with the $n$-element set
$\{1,2,\ldots,n\}$. In this way we write $\Nat$ for the category with
natural numbers $n$ as objects and with functions $n\rightarrow m$
between them. Thus there is a full and faithful functor
$\Nat\hookrightarrow\Sets$. This category $\Nat$ has finite products,
with final object $1$ and binary product $n\times m$ given by
multiplication of numbers.

We mostly apply the above functors $\Mlt, \Dst$ to $n\in\Nat$, as
sets. Then:
\begin{itemize}
\item $\Mlt(n) \cong \NNO^{n}$, via $\sum_{i}\alpha_{i}\ket{i} \mapsto
  (\alpha_{1}, \ldots, \alpha_{n})$. We shall often (implicitly) use
  this isomorphism. So why use special notation $\Mlt(n)$ for
  $\NNO^{n}$? Because $\Mlt$ is a functor and its action $\Mlt(h)$ on
  a function $h$ is most useful.


\item $\Dst(n+1)$ can be identified with the $n$-simplex of $r_{0},
  \ldots, r_{n} \in [0,1]$ with $\sum_{i}r_{i} = 1$. Via this
  identification we consider $\Dst(n)$ a subset of $\R^{n+1}$, where
  $\Dst(1) \cong 1 = \{1\} \hookrightarrow \R^{1} \cong \R$.
\end{itemize}

\noindent Having seen the basics of $\Mlt$ and $\Dst$, we consider the
following variations.
\begin{enumerate}
\item We write $\neMlt(X) \hookrightarrow \Mlt(X)$ for the subset of
  \emph{non-empty} multisets, that is, of multisets with non-empty
  support.  More explicitly, $\neMlt(X)$ contains those multisets
  $\sum_{i}\alpha_{i}\ket{x_{i}}$ with $\alpha_{i} > 0$ for some index
  $i$; alternatively, the sum $\sum_{i}\alpha_{i}$ is
  non-zero. Distributions have non-empty support by definition.

\item We further write $\aneMlt(X) \hookrightarrow \neMlt(X)$ for the
  set of multisets $\sum_{i}\alpha_{i}\ket{x_{i}}$ with `full
  support', that is with $\supp(\phi) = X$. This means that $X =
  \{x_{1}, \ldots, x_{n}\}$ and $\alpha_{i} > 0$ for all $i$. Using
  $\aneMlt(X)$ only makes sense for finite sets $X$. We write
  $\aneDst(X) \hookrightarrow \Dst(X)$ for the subset of distribitions
  with full support.

Functoriality is a bit subtle for $\aneMlt$ and $\aneDst$. The
descriptions $\Mlt(h)$ only makes sense for \emph{surjective}
functions $h$.
\end{enumerate}

A \emph{channel} $c\colon\kto{X}{Y}$ is a function $c\colon X
\rightarrow \Dst(Y)$. It gives a probability distribution
$c(x)\in\Dst(Y)$ on $Y$ for each element $x\in X$. It captures the
idea of a \emph{conditional} probability distribution $p(y\mid x)$.
Given a distribution $\omega\in\Dst(X)$ on the domain $X$ of a channel
$c\colon \kto{X}{Y}$ we write $c \gg \omega \in \Dst(Y)$ for the
distribution on $Y$ that is obtained by `state transformation':
\[ \begin{array}{rcl}
(c \gg \omega)(y)
& \coloneqq &
\sum_{x} c(x)(y)\cdot \omega(x).
\end{array} \]

\noindent Given another channel $d \colon \kto{Y}{Z}$ we write $d
\klafter c \colon \kto{X}{Z}$ for the composite channel defined by
$(d\klafter c)(x) \coloneqq d \gg c(x)$.

From a joint distribution $\omega\in\Dst(X\times Y)$ one can extract
a channel $c\colon \kto{X}{Y}$,
\begin{equation}
\label{DisintegrationEqn}
\begin{array}{rcccl}
c(x)(y)
& \coloneqq &
\displaystyle\frac{\omega(x,y)}{\fM(\omega)(x)}
& = &
\displaystyle\frac{\omega(x,y)}{\sum_{y}\omega(x,y)}.
\end{array}
\end{equation}

\noindent This channel exists if $\fM(\omega)(x) > 0$ for each $x\in
X$, that is, if the first marginal $\fM(\omega)$ has full support,
\ie~is in $\aneDst(X)$. This extracted channel, if it exists, is
unique with the property $\omega = \tuple{\idmap, c} \gg \fM(\omega)$,
where $\tuple{\idmap,c} \colon \kto{X}{X\times Y}$ is the channel with
$\tuple{\idmap,c}(x) = \sum_{y} c(x)(y)\ket{x,y}$. We shall use
disintegration as a partial function $\dis \colon \Dst(X\times Y)
\rightarrow \Dst(Y)^{X}$. See~\cite{ClercDDG17,ChoJ17a} for more
information.


For a distribution $\omega\in\Dst(X)$ on a set $X$ and a (fuzzy)
predicate $p\colon X \rightarrow [0,1]$ on the same set $X$ we write
$\omega\models p$ for the validity (or expected value) of $p$ in
$\omega$. It is the number in $[0,1]$ defined as $\sum_{x\in X}
\omega(x)\cdot p(x)$. In case this validity is nonzero, we write
$\omega|_{p} \in\Dst(X)$ for the conditioned distribution, updated
with predicate $p$. It is defined as $\omega|_{p}(x) =
\frac{\omega(x)\cdot p(x)}{\omega\models p}$. For more details,
see~\cite{JacobsZ16,JacobsZ18,Jacobs17a,Jacobs15d}.

\section{Frequentist learning by counting}\label{FrequentistSec}

As mentioned in the introduction, maximal likelihood estimation (MLE)
is one kind of parameter learning, see
\textit{e.g.}~\cite{KollerF09,Darwiche09,JensenN07}. We reframe it
here as frequentist learning. Our categorical reformulation for
discrete probability distributions (multinomials) uses the non-empty
multiset functor $\neMlt$ and the distribution functor $\Dst$ from
Section~\ref{DiscPrereqSec}. It turns out that the process of
learning-by-counting involves some basic categorical structure: it is
a monoidal natural transformation, that can be applied locally.

\begin{definition}
\label{FrequentistDef}
For each $n\in\NNO$ we define (discrete) maximal likelihood estimation
as a function $\ell_{n} \colon \neMlt(n) \rightarrow \Dst(n)$,
determined by:
\begin{equation}
\label{FrequentistEqn}
\begin{array}{rclcrcl}
\ell_{n}(\alpha_{1}, \ldots, \alpha_{n})
& \coloneqq &
\frac{\alpha_{1}}{\alpha}\ket{1} + \cdots +
   \frac{\alpha_{n}}{\alpha}\ket{n}
& \quad\mbox{where}\quad &
\alpha 
& \coloneqq &
\sum_{i}\alpha_{i}.
\end{array}
\end{equation}
\end{definition}

The map $\ell$ turns numbers $\alpha_{i}\in\NNO$ of occurrences of
data items $i$ into a distribution, essentially by normalisation, as
we have seen earlier for table~\eqref{TableEqn}. The
distribution~\eqref{FrequentistEqn} yields a maximum for a likelihood
function on distributions, see
Remark~\ref{FrequentistRem}~\eqref{FrequentistRemName} below. Here we
are interested in its categorical properties.


\begin{lemma}
\label{FrequentistNatLem}
The maps $\ell_{n} \colon \neMlt(n) \rightarrow \Dst(n)$ form a
natural transformation as on the left below.
\vspace*{-1em}
\[ \vcenter{\xymatrix@C+2pc{
\Nat\rtwocell^{\neMlt}_{\Dst}{\ell} & \Sets
}}
\hspace*{6em}
\vcenter{\xymatrix@C+1pc@R-1pc{
\neMlt(n_{1}\times n_{2})\ar[d]_{\ell_{n_{1}\times n_{2}}}\ar[r]^-{\neMlt(\pi_{i})}
  & \neMlt(n_{i})\ar[d]^{\ell_{n_i}}
\\
\Dst(n_{1}\times n_{2})\ar[r]^-{\Dst(\pi_{i})} & \Dst(n_{i})
}} \]

\noindent In particular, maximal likelihood estimation $\ell$ commutes
with marginalisations $\neMlt(\pi_{i})$ and $\Dst(\pi_{i})$, obtained
via projections $\pi_{i} \colon n_{1}\times n_{2}\rightarrow n_{i}$,
see the naturality diagram, above on the right.
\end{lemma}


\begin{proof}
Let $h\colon n \rightarrow m$ be a function (between numbers as sets).
Then:
\[ \begin{array}[b]{rcl}
\big(\ell_{m} \after \Mlt(h)\big)(\sum_{i}\alpha_{i}\ket{i})
& = &
\ell_{m}\big(\sum_{i}\alpha_{i}\ket{h(i)}\big)
\\
& = &
\sum_{i} \frac{\alpha_i}{\alpha}\ket{h(i)}
\\
& = &
\Dst(h)\big(\sum_{i} \frac{\alpha_i}{\alpha}\ket{i}\big)
\hspace*{\arraycolsep}=\hspace*{\arraycolsep}
\big(\Dst(h) \after \ell_{n}\big)(\sum_{i}\alpha_{i}\ket{i}).
\end{array} \eqno{\QEDbox} \]
\end{proof}

\begin{remark}
\label{FrequentistRem}
\begin{enumerate}
\item \label{FrequentistRemMonad} The maps $\ell_{n} \colon \neMlt(n)
  \rightarrow \Dst(n)$ do \emph{not} form a map of monads, since they
  do not commute with multiplications. Here is a simple
  counterexample. Consider the multiset of multsets $\Phi =
  1\bigket{\,2\ket{a} + 4\ket{c}\,} + 2\bigket{\, 1\ket{a} + 1\ket{b}
    + 1\ket{c}\,}$ in $\neMlt(\neMlt(\{a,b,c\}))$. First taking
  (multiset) multiplication, and then normalisation gives $\Phi
  \longmapsto 4\ket{a} + 2\ket{b} + 6\ket{c} \longmapsto
  \frac{1}{3}\ket{a} + \frac{1}{6}\ket{b} +
  \frac{1}{2}\ket{c}$. However, first (outer en inner) normalising and
  then doing (distribution) multiplication yields: $\Phi \longmapsto
  \frac{1}{3}\bigket{\,\frac{1}{3}\ket{a} + \frac{2}{3}\ket{c}\,} +
  \frac{2}{3}\bigket{\, \frac{1}{3}\ket{a} + \frac{1}{3}\ket{b} +
    \frac{1}{3}\ket{c}\,} \longmapsto \frac{1}{3}\ket{a} +
  \frac{2}{9}\ket{b} + \frac{4}{9}\ket{c}$. Consequently, the
  $\ell_n$'s do not form a map between Kleisli categories.

\item \label{FrequentistRemName} We briefly explain the name `maximal
  likelihood estimation'. Given data $\vec{\alpha} = (\alpha_{1},
  \ldots, \alpha_{n})\in\neMlt(n)$ we can define a likelihood function
  $L_{\vec{\alpha}} \colon \Dst(n) \rightarrow [0,1]$ by
  $\sum_{i}r_{i}\ket{i} \mapsto \prod_{i} r_{i}^{\alpha_i}$. It
  captures the validity of getting $\alpha_{i}$ times outcome $i$, for
  each $i$, see the next paragraph.  This $L_{\vec{\alpha}}$ takes its
  maximum at the distribution $\ell_{n}(\alpha_{1}, \ldots,
  \alpha_{n}) = \sum_{i} \frac{\alpha_i}{\alpha}\ket{i}$, see
  \textit{e.g.}~\cite{KollerF09} for further details.


\end{enumerate}
\end{remark}

In Section~\ref{PrelimSec} we mentioned that one can extract a
conditional probability table (or channel) either from the table of
data, or from the associated probability distribution. We are going to
make this precise next. This requires some additional notation.  For
multisets there is an isomorphism:
\begin{equation}
\label{RowExtractEqn}
\xymatrix{
\Mlt(n\times m)\ar[r]^-{\row}_-{\cong} & \Mlt(m)^{n}
\quad\mbox{via}\quad
{\begin{array}{rcl}
\row\Big(\sum_{ij}\alpha_{ij}\ket{ij}\Big)(i)
& \coloneqq &
\sum_{j} \alpha_{ij}\ket{j}.
\end{array}}
}
\end{equation}

\noindent This $\row$ function captures disintegration for
multisets. It is much simpler than for distributions, since there is
no normalisation involved.

In the current situation we need to adapt the type of this
$\row$ function because we wish to exclude empty multisets.  Therefor
we use ad hoc notation $\sneMlt(n\times m)$ for the collection of
multisets $\phi = \sum_{ij}\alpha_{ij}\ket{ij}$ for which
$\Mlt(\pi_{1})(\phi)(i) = \sum_{j}\alpha_{ij} \neq 0$ for each $i\in
n$. Then $\sneMlt(n\times m) \subseteq \neMlt(n\times m)$. We can then
restrict the row-extraction function to $\row \colon \sneMlt(n\times
m) \rightarrow \neMlt(m)^{n}$. We will next show how this
row-extraction gives a link with disintegration.

\begin{proposition}
\label{FrequentistTableProp}
Disintegrations for multisets and for distributions commute
with maximal likelihood estimation $\ell$, as in:
\[ \xymatrix@R-1pc{
\sneMlt(n\times m)\ar[d]_{\ell_{n\times m}}\ar[r]^-{\row} &
   \neMlt(m)^{n}\ar[d]^{(\ell_{m})^{n}}
\\
\Dst(n\times m)\ar[r]_{\dis} & \Dst(m)^{n}
} \]
\end{proposition}

\begin{proof}
We simply compute:
\[ \begin{array}[b]{rcl}
\big(\dis \after \ell_{n\times m}\big)(\sum_{ij}\alpha_{ij}\ket{ij})(i)
& = &
\dis\big(\sum_{ij}\frac{\alpha_{ij}}{\alpha}\ket{ij}\big)(i)
   \quad \mbox{where } \alpha = \sum_{ij}\alpha_{ij}
\\
& \smash{\stackrel{\eqref{DisintegrationEqn}}{=}} &
\sum_{j} \frac{\nicefrac{\alpha_{ij}}{\alpha}}
             {\sum_{j}\nicefrac{\alpha_{ij}}{\alpha}}\ket{j}
\\
& = &
\sum_{j} \frac{\alpha_{ij}}{\sum_{j}\alpha_{ij}}\ket{j}
\\
& = &
\ell_{m}\big(\sum_{j} \alpha_{ij}\ket{j}\big)
\\
& \smash{\stackrel{\eqref{RowExtractEqn}}{=}} &
\ell_{m}\big(\row(\sum_{ij}\alpha_{ij}\ket{ij})(i)\big)
\\
& = &
\big((\ell_{m})^{n} \after \row\big)(\sum_{ij}\alpha_{ij}\ket{ij})(i).
\end{array} \eqno{\QEDbox} \]
\end{proof}

Table~\eqref{TableEqn} in Section~\ref{PrelimSec} can be
described\footnote{Here, in a more abstract mode, we write $2$ instead
  of $\{H,T\}$ and $3$ instead of $\{0,1,2\}$.} as an element of
$\sneMlt(2\times 3)$, from which we have shown that we can extract a
channel $\kto{2}{3}$. The above result says that it does not matter if
we form the corresponding joint distribution first and then
disintegrate, or if we extract the CPT/channel directly from the
table. This sometimes called a `decomposition property' that allows us
to reduce learning of a CPT/channel to a set of local learning
problems. The `local' adjective means: by $\ell_{m}$, under the
exponent $n$, in the diagram in
Proposition~\ref{FrequentistTableProp}. An alternative and equivalent
way to express the diagram in Proposition~\ref{FrequentistTableProp}
is as equation:
\begin{equation}
\label{FrequentistTableEqn}
\begin{array}{rcl}
\ell_{n\times m}(\vec{\alpha})
& = &
\tuple{\idmap, (\ell_{m})^{n}(\row(\vec{\alpha}))} \gg 
   \ell_{n}\big(\Mlt(\pi_{1})(\vec{\alpha})\big).
\end{array}
\end{equation}

\auxproof{
Alternatively:
\[ \begin{array}{rcl}
\lefteqn{\tuple{\idmap, (\ell_{m})^{n}(\row(\vec{\alpha}))} \gg 
   \ell_{n}\big(\Mlt(\pi_{1})(\vec{\alpha})\big)}
\\
& = &
\displaystyle\sum_{i,j} \ell_{n}\big(\Mlt(\pi_{1})(\vec{\alpha})\big)(i) \cdot
   (\ell_{m})^{n}(\row(\vec{\alpha}))(i)(j)\bigket{ij}
\\
& = &
\displaystyle\sum_{i,j} \frac{\sum_{j}\alpha_{ij}}{\alpha} \cdot
   \ell_{m}\big(\textstyle\sum_{j}\alpha_{ij}\ket{j}\big)\bigket{ij}
\\
& = &
\displaystyle\sum_{i,j} \frac{\sum_{j}\alpha_{ij}}{\alpha} \cdot
   \frac{\alpha_{ij}}{\sum_{j}\alpha_{ij}}\bigket{ij}
\\
& = &
\displaystyle\sum_{i,j} \frac{\alpha_{ij}}{\alpha}\bigket{ij}
\\
& = &
\ell_{n\times m}(\vec{\alpha})
\end{array} \]
}

Thus, the following (categorical) picture emerges. Bayesian networks
can be seen as graphs in the Kleisli category $\Kl(\Dst)$ of the
distribution monad $\Dst$, see esp.~\cite[Chap.~4]{Fong12} and
also~\cite{JacobsZ16,JacobsZ18}. If we write $G$ for the underlying
graph of a Bayesian network, its conditional probability tables may be
described either by:
\begin{enumerate}
\item a graph homomorphism $G \rightarrow
  \mathcal{U}\big(\Kl(\Dst)\big)$, where
  $\mathcal{U}\big(\Kl(\Dst)\big)$ is the underlying graph of the
  Kleisli category $\Kl(\Dst)$;

\item a strong monoidal functor $\text{Free}(G) \rightarrow \Kl(\Dst)$
  from the free monoidal category $\text{Free}(\mathcal{G})$ on
  $\mathcal{G}$ with diagonals and discarders.
\end{enumerate}

\noindent Since nodes of Bayesian networks typically have finite sets
as domains, we can restrict to the full subcategory $\KlN(\Dst)
\hookrightarrow \Kl(\Dst)$ with $n\in\NNO$ as objects.

Interestingly, a frequency table as in~\eqref{TableEqn} may be
reorganised similarly as a graph homomorphism $G \rightarrow
\mathcal{U}\big(\KlN(\Mlt)\big)$, for the multiset monad $\Mlt$. This
reorganisation happens via the above `$\row$' function and
marginalisation.

We have seen in
Remark~\ref{FrequentistRem}~\eqref{FrequentistRemMonad} that the
$\ell_{n} \colon \neMlt(n) \rightarrow \Dst(n)$ do \emph{not} form a
map of monads, and hence do not produce a functor $\KlN(\neMlt)
\rightarrow \KlN(\Mlt)$. But these $\ell_n$ do form a graph
homomorphism between the underlying graphs. Hence the passage from a
frequency table to a Bayesian network (with underlying graph
$G$) can be described as a composite of graph homomorphisms:
\[ \xymatrix@C-1pc{
G\ar[rrr]^-{\text{table}} & & &
   \mathcal{U}\big(\KlN(\Mlt)\big)\ar[rr]^-{\ell} & &
   \mathcal{U}\big(\KlN(\Dst)\big)
} \]


\auxproof{
We conclude with one more observation. It requires `monoidal
structure' for the functors $\Mlt$ and $\Dst$. It is described in
terms of a product map $\otimes \colon \Mlt(X) \times \Mlt(Y)
\rightarrow \Mlt(X\times Y)$ given by $(\sum_{i}\alpha_{i}\ket{x_i})
\otimes (\sum_{j}\beta_{j}\ket{y_j}) = \sum_{ij} \alpha_{i}\cdot
\beta_{j}\ket{x_{i},x_{j}}$. A similar product map $\otimes\colon
\Dst(X)\times\Dst(Y) \rightarrow \Dst(X\times Y)$ exists for
probability distributions.
}

We do not elaborate the following observation since it is not used any
further.

\begin{lemma}
\label{FrequentistMonoidalLem}
The natural transformation $\ell \colon \neMlt \Rightarrow \Dst$
from~\eqref{FrequentistEqn} is monoidal. \QED
\end{lemma}

\auxproof{
\begin{proof}
For multisets $\sum_{i}\alpha_{i}\ket{i}\in\neMlt(n)$ and
$\sum_{j}\beta_{j}\ket{j}\in\neMlt(m)$ write $\alpha = \sum_{i}\alpha_{i}$
and $\beta = \sum_{j}\beta_{j}$. Then:
\[ \begin{array}[b]{rcl}
\lefteqn{\textstyle\ell_{n\times m}\big((\sum_{i}\alpha_{i}\ket{i}) 
   \otimes (\sum_{j}\beta_{j}\ket{j})\big)}
\\
& = &
\ell_{n\times m}\big(\sum_{ij} \alpha_{i}\cdot\beta_{j}\ket{ij}\big)
\\
& = &
\sum_{ij} \frac{\alpha_{i}\cdot\beta_{j}}{\alpha\cdot\beta}\ket{ij} 
\qquad\qquad \mbox{since } \sum_{ij}\alpha_{i}\cdot\beta_{j} = 
   (\sum_{i} \alpha_{i})\cdot (\sum_{j}\beta_{j}) = \alpha\cdot\beta
\\
& = &
(\sum_{i} \frac{\alpha_{i}}{\alpha}\ket{i}) \otimes 
   (\sum_{j} \frac{\beta_{j}}{\beta}\ket{j})
\\
& = &
\ell_{n}(\sum_{i}\alpha_{i}\ket{i}) \otimes \ell_{m}(\sum_{j}\beta_{j}\ket{j}).
\end{array} \]

\noindent We still have to check that $\ell_{1} \colon \neMlt(1)
\rightarrow \Dst(1)$ commutes with the units $1\rightarrow\neMlt(1)$
and $1\rightarrow\Dst(1)$. But this is trivial, since $\Dst(1) \cong
1$ is final. \QED
\end{proof}
}

This result says that likelihood estimation yields the same outcome on
separate tables as on a parallel combination of these tables. It is
mostly of categorical relevance. In a learning scenario there is no
point in first combining two separate tables into one, and then
learning from the combination.

\section{Prerequisites on continuous probability}\label{ContPrereqSec}

Frequentist learning can be done via discrete probability
distributions, but Bayes\-ian learning requires continuous probability
distributions --- in particular in the form of Dirichlet distributions
on probabilistic parameters in $[0,1]$ for the multinomial
case. Categorically, such continuous distributions are captured via
the Giry monad $\Giry$ on the category of measurable
spaces~\cite{Giry82,Panangaden09,Jacobs17a}. Here we shall only use
measurable spaces which are a bounded subset of $\R^n$, for some $n$,
with their standard measure. Hence we do not need measure theory in
full generality.

We recall that on a measurable space $X = (X,\Sigma_{X})$, with
measurable subsets $\Sigma_{X} \subseteq \Pow(X)$, the Giry monad
$\Giry$ is defined as:
\[ \begin{array}{rcl}
\Giry(X)
& \coloneqq &
\set{\omega\colon \Sigma_{X} \rightarrow [0,1]}{\omega 
   \mbox{ is countably additive, and } \omega(X) = 1}.
\end{array} \]

\noindent A function $f\colon X \rightarrow Y$ between measurable
spaces is called measurable if its inverse image function $f^{-1}
\colon \Pow(Y) \rightarrow \Pow(X)$ restricts to $f^{-1} \colon
\Sigma_{Y} \rightarrow \Sigma_{X}$. For such a measurable $f$ one
gets a measurable function $\Giry(f) \colon \Giry(X) \rightarrow
\Giry(Y)$ by:
\begin{equation}
\label{GiryFunEqn}
\begin{array}{rcl}
\Giry(f)\Big(\Sigma_{X}\xrightarrow{\omega}[0,1]\Big)
& \coloneqq &
\Big(\Sigma_{Y} \xrightarrow{f^{-1}} \Sigma_{X} \xrightarrow{\omega}[0,1]\Big).
\end{array}
\end{equation}

\noindent Thus, $\Giry(f)(\omega)$ is the `image' measure, also called
the `push forward' measure.

As mentioned, in the current context the measurable space $X$ is
typically a (bounded) subset $X \subseteq \R^{n}$ for some $n$. The
probability measures in $\Giry(X)$ that we consider are given by a
`probability density function' (abbreviated as `pdf'). Such a pdf is a
function $f \colon X \rightarrow \nnR$ with $\int f(\vec{x})\intd
\vec{x} = 1$. We write $\PDF(X)$ for the set of pdf's on
$X\subseteq\R^{n}$, for some $n$.

For a measurable subset $M\subseteq X$ we write $\int_{M}
f(\vec{x})\intd \vec{x}$ or $\int_{\vec{x}\in M} f(\vec{x})\intd
\vec{x}$ for the integral $\int \indic{M}(\vec{x})\cdot
f(\vec{x})\intd \vec{x}$, where $\indic{M} \colon X \rightarrow [0,1]$
is the indicator function for $M$; it is $1$ on $\vec{x}\in M$ and $0$
on $\vec{x}\not\in M$.  In this way we can define a function:
\begin{equation}
\label{PDFtoGiryEqn}
\xymatrix{
\PDF(X)\ar[r]^-{\mathcal{I}} & \Giry(X)
\qquad\mbox{namely}\qquad
{\begin{array}{rcl}
\mathcal{I}(f)(M)
& \coloneqq &
\displaystyle\int_{M} f(\vec{x}) \intd \vec{x}.
\end{array}}
}
\end{equation}

Disintegration for continuous joint distributions is much more
difficult than for discrete distributions. However, in our current
setting (see also~\cite{ClercDDG17,ChoJ17a}), where we restrict
ourselves to distributions $\omega \in\Giry(X\times Y)$ given by a
pdf, say $\omega = \mathcal{I}(f) = \int f$, for $f\colon X\times
Y\rightarrow \nnR$, there is a formula to obtain $\dis(\omega) \in
\Giry(Y)^{X}$, namely:
\begin{equation}
\label{GiryDisintegrationEqn}
\begin{array}{rcl}
\dis\big(\mathcal{I}(f)\big)(x)(N)
& \coloneqq &
\displaystyle\frac{\int_{N} f(x,y)\intd y}{\int f(x,y)\intd y}.
\end{array}
\end{equation}

For a continuous distribution $\omega\in\Giry(X)$ on a measurable
space $X$ and a measurable function $q\colon X \rightarrow [0,1]$ we
write $\omega\models q$ for the `continuous' validity value in
$[0,1]$, obtained via Lebesgue integration $\int q\intd \phi$. If the
measure $\phi$ is given by a pdf $f$, as in $\phi = \mathcal{I}(f) =
\int f$, then $\int q \intd \phi$ equals $\int q(\vec{x})\cdot
f(\vec{x}) \intd\vec{x}$. Also in the continuous case we can update a
distribution $\omega\in\Giry(X)$ to $\omega|_{p}\in\Giry(X)$ via the
definition:
\begin{equation}
\label{ContConditionEqn}
\begin{array}{rcl}
\omega|_{p}(M)
& \coloneqq &
\displaystyle\frac{\int_{M} q\intd \omega}{\omega\models q}.
\end{array}
\end{equation}

\section{Bayesian learning}\label{BayesianSec}

In Section~\ref{FrequentistSec} we have uncovered some basic
categorical structure in frequentist parameter estimation. Our next
challenge is to see if we can find similar structure for Bayesian
parameter estimation, where parameters (for discrete probability
distributions, or multinomials) are obtained via successive Bayesian
updates.  This will involve (continuous) Dirichlet distributions over
the probabilistic parameters $r_{i}\in[0,1]$ of distributions
$\sum_{i}r_{i}\ket{x_i}$. For Dirichlet distributions it is required
that each of these $r_{i} \in [0,1]$ is non-zero. We shall write
$\aneDst(n) \subseteq \Dst(n) \subseteq \R^n$ for the subset of
$(r_{1}, \ldots, r_{n})$ with $r_{i} > 0$ for \emph{each} $i\in n$.


We shall study \emph{Dirichlet} distributions via their pdf's
written as $d_n$ in:
\[ \xymatrix{
\aneMlt(n) \ar[r]^-{d_n} & \PDF(\aneDst(n))
} \]

\noindent given by:
\begin{equation}
\label{DirichletPDFEqn}
\begin{array}{rcl}
d_{n}(\alpha_{1}, \ldots, \alpha_{n})(x_{1}, \ldots, x_{n})
& \coloneqq &
{\displaystyle\frac{\Gamma(\sum_{i}\alpha_{i})}
                   {\prod_{i} \Gamma(\alpha_{i})}} 
    \cdot \prod_{i} x_{i}^{\alpha_{i}-1}.
\end{array}
\end{equation}

\noindent The variables $\alpha_{i} \in \pNNO$ are sometimes called
\emph{hyperparameters}, in contrast to the `ordinary' variables
$x_{i}\in (0,1]$. We use these Dirichlet pdf's $d_n$ to define what we
  call the Dirichlet distribution functions $\Dir_n$ in:
\begin{equation}
\label{DirichletDistributionEqn}
\xymatrix{
\aneMlt(n)\ar[r]^-{\Dir_n} & \Giry\big(\aneDst(n)\big)
\qquad\mbox{via}\qquad
{\begin{array}{rcl}
\Dir_{n}
& \coloneqq &
\mathcal{I} \after d_{n}.
\end{array}}
}
\end{equation}

\noindent Hence $\Dir_{n}(\vec{\alpha})(M) =
\mathcal{I}\big(d_{n}(\vec{\alpha})\big) = \int_{M}
d_{n}(\vec{\alpha})(\vec{x})\intd \vec{x}$, see~\eqref{PDFtoGiryEqn}.

Our main aim in this section is to prove that the maps $\Dir_{n}
\colon \aneMlt(n) \rightarrow \Giry\big(\aneDst(n)\big)$ are natural
in $n$ --- in analogy with naturality for $\ell_{n} \colon \neMlt(n)
\rightarrow \Dst(n)$ in Lemma~\ref{FrequentistNatLem} --- but
\emph{only} for surjective functions. We proceed via a functorial
reformulation of the aggregation property of the Dirichlet functions.

There are some basic facts that we need about these $d_n$,
see~\cite{Bishop06} for details.
\begin{enumerate}
\item The operation $\Gamma$ in~\eqref{DirichletPDFEqn} is the `Gamma'
  function, which is defined on natural numbers $k > 1$ as $\Gamma(k)
  = (k-1)!$. Hence $\Gamma$ can be defined recursively as $\Gamma(1) =
  1$ and $\Gamma(k+1) = k\cdot\Gamma(k)$. We shall use $\Gamma$ on
  natural numbers only, but it can be defined on complex numbers too.

\item The fraction in~\eqref{DirichletPDFEqn} works as a normalisation
  factor, and satisfies:
\begin{equation}
\label{DirichletNormalisationEqn}
\begin{array}{rcl}
\displaystyle\frac{\prod_{i} \Gamma(\alpha_{i})}{\Gamma(\sum_{i}\alpha_{i})}
& = &
\displaystyle\int_{\vec{x}\in\Dst(n)} \textstyle
    \prod_{i} x_{i}^{\alpha_{i}-1} \intd \vec{x}.
\end{array}
\end{equation}

\noindent This ensures that $d_{n}(\vec{\alpha})$ is a pdf.


\item For each $\vec{\alpha}\in\aneMlt(n)$ and $i\in n$ one has:
\begin{equation}
\label{DirichletPointEqn}
\begin{array}{rcccl}
\displaystyle\int_{\vec{x}\in\aneDst(n)} x_{i}\cdot d_{n}(\vec{\alpha})(\vec{x}) 
   \intd \vec{x}
& = &
\displaystyle \frac{\alpha_{i}}{\sum_{j}\alpha_j}
& \smash{\stackrel{\eqref{FrequentistEqn}}{=}} &
\ell_{n}(\vec{\alpha})(i).
\end{array}
\end{equation}

\noindent The first equation follows easily from the previous two
points.  The second equation establishes a (standard) link between
maximal likelihood estimation $\ell_n$ and Dirichlet pdf's $d_n$.

\auxproof{
\[ \begin{array}{rcl}
\displaystyle\int_{\vec{x}\in\Dst(n)} x_{i}\cdot d_{n}(\vec{\alpha})(\vec{x}) 
   \intd \vec{x}
& = &
\displaystyle\int_{\vec{x}\in\Dst(n)} 
   {\displaystyle\frac{\Gamma(\sum_{j}\alpha_{j})}{\prod_{j} \Gamma(\alpha_{j})}} 
   \cdot x_{i}^{\alpha_{i}} \cdot \prod_{j\neq i} x_{j}^{\alpha_{j}-1} \intd \vec{x}
\\
& \smash{\stackrel{\eqref{DirichletNormalisationEqn}}{=}} &
\displaystyle\frac{\Gamma(\sum_{j}\alpha_{j})}{\prod_{j} \Gamma(\alpha_{j})} 
   \cdot \frac{\Gamma(\alpha_{i}+1)\cdot\prod_{j\neq i}\Gamma(\alpha_{j})}
   {\Gamma(1 + \sum_{j}\alpha_{j})}
\\
& = &
\displaystyle\frac{\Gamma(\sum_{j}\alpha_{j})}{\prod_{j} \Gamma(\alpha_{j})} 
   \cdot \frac{\alpha_{i} \cdot \Gamma(\alpha_{i}) \cdot 
       \prod_{j\neq i}\Gamma(\alpha_{j})}
   {(\sum_{j}\alpha_{j}) \cdot \Gamma(\sum_{j}\alpha_{j})}
\\
& = &
\displaystyle\frac{\alpha_{i}}{\sum_{j}\alpha_{j}}.
\end{array} \]
}
\end{enumerate}

\noindent Whereas results about maximal likelihood estimation in
Section~\ref{FrequentistSec} are relatively easy, things are
mathematically a lot more challenging now. We need the so-called
\emph{aggregation property} of the Dirichlet pdf's $d_n$,
from~\cite{FrigyikKG10}. Our contribution is a reformulation this
property --- not the property itself --- in functorial form, starting
from a formulation in the first point below that is close to the
orginal.

\begin{lemma}
\label{OneSumLem}
Let $(\alpha_{1}, \ldots, \alpha_{n}) \in \aneMlt(n)$.
\begin{enumerate}
\item For $(x_{2}, \ldots, x_{n}) \in \aneDst(n-1)$,
\begin{equation}
\label{OneSumEqn} 
\begin{array}{rcl}
\lefteqn{d_{n-1}(\alpha_{1}+\alpha_{2}, \alpha_{3}, \ldots, \alpha_{n})
   (x_{2}, x_{3}, \ldots, x_{n})}
\\
& = &
\displaystyle\int_{y\in (0,x_{2})} 
   d_{n}(\alpha_{1}, \alpha_{2}, \alpha_{3}, \ldots, \alpha_{n})
   (y, x_{2} - y, x_{3}, \ldots, x_{n}) \intd y.
\end{array}
\end{equation}

\item Define $h\colon n \twoheadrightarrow n-1$ as $h(1) = h(2) = 1$
  and $h(i+2) = i+1$, so that $\aneMlt(h)(\alpha_{1}, \ldots,
  \alpha_{n}) = (\alpha_{1}+\alpha_{2}, \alpha_{3}, \ldots,
  \alpha_{n})$. For a measurable subset $N\subseteq \aneDst(n-1)$,
\begin{equation}
\label{OneSumSubsetEqn} 
\begin{array}{rcl}
\displaystyle\int_{\vec{x}\in N} d_{n-1}\Big(\aneMlt(h)(\vec{\alpha})\Big)
   (\vec{x}) \intd \vec{x}
& = &
\displaystyle\int_{\vec{y}\in\aneDst(h)^{-1}(N)} 
   d_{n}(\vec{\alpha})(\vec{y}) \intd \vec{y}.
\end{array}
\end{equation}
\end{enumerate}
\end{lemma}

\begin{proof}
\begin{enumerate}
\item By expanding the definition~\eqref{DirichletPDFEqn} of $d$, the
left-hand-side and right-hand-side of equation~\eqref{OneSumEqn}
become:

\smallskip

\begin{itemize}
\item $\displaystyle\frac{\Gamma(\sum_{i}\alpha_{i})}
   {\Gamma(\alpha_{1}+\alpha_{2})\cdot \prod_{i>2}\Gamma(\alpha_{i})} \cdot
   \textstyle x_{2}^{\alpha_{1}+\alpha_{2}-1}\cdot \prod_{i>2}x_{i}^{\alpha_{i}-1}$

\medskip

\item $\displaystyle\int_{0}^{x_{2}} \frac{\Gamma(\sum_{i}\alpha_{i})}
   {\prod_{i}\Gamma(\alpha_{i})} \cdot \textstyle
   y^{\alpha_{1}-1} \cdot (x_{2}-y)^{\alpha_{2}-1}\cdot \prod_{i>2}x_{i}^{\alpha_{i}-1}$
\end{itemize}

\smallskip

\noindent By eliminating the same factors on both sides, what
we have to prove is:
\[ \begin{array}{rcl}
\displaystyle\frac{\Gamma(\alpha_{1}) \cdot \Gamma(\alpha_{2})}
   {\Gamma(\alpha_{1}+\alpha_{2})} \cdot x^{\alpha_{1}+\alpha_{2}-1}
& = &
\displaystyle \int_{y\in (0,x)} y^{\alpha_{1}-1} \cdot (x-y)^{\alpha_{2}-1} \intd y.
\end{array} \]


\noindent This equation follows by using a suitable substitution $s$ in:
\[ \begin{array}[b]{rcl}
\displaystyle\frac{\Gamma(\alpha_{1}) \cdot \Gamma(\alpha_{2})}
      {\Gamma(\alpha_{1}+\alpha_{2})}
& \smash{\stackrel{\eqref{DirichletNormalisationEqn}}{=}} &
\displaystyle\int_{0}^{1} \textstyle
    z^{\alpha_{1}-1}\cdot (1-z)^{\alpha_{2}-1} \intd z
\\
& = &
\displaystyle\int_{s(0)}^{s(x)} \textstyle
    z^{\alpha_{1}-1}\cdot (1-z)^{\alpha_{2}-1} \intd z
   \qquad \mbox{for } s(y) = \frac{y}{x}
\\[+0.8em]
& = &
\displaystyle\int_{0}^{x} \textstyle
    s(y)^{\alpha_{1}-1}\cdot (1-s(y))^{\alpha_{2}-1} \cdot s'(y) \intd y
\\[+0.8em]
& = &
\displaystyle\int_{0}^{x} \textstyle
    (\frac{y}{x})^{\alpha_{1}-1}\cdot (\frac{x-y}{x})^{\alpha_{2}-1} \cdot 
    \frac{1}{x} \intd y
\\
& = &
\displaystyle\frac{1}{x^{\alpha_{1}+\alpha_{2}-1}}\cdot\int_{0}^{x} \textstyle
    y^{\alpha_{1}-1}\cdot (x-y)^{\alpha_{2}-1} \intd y.
\end{array} \]

\auxproof{
An earlier proof attempt without using the normalisation equation
but using the binomial theorem resulted in a proof obligation:
\[ \begin{array}{rcl}
\displaystyle\frac{1}{n\cdot {{m+n}\choose n}}
& = &
\displaystyle\sum_{k=0}^{k=m} {m \choose k}\cdot 
   \frac{(-1)^{k}}{n+k}
\end{array} \]

\noindent I have tested this for arbitrary $n,m$ in Python, and it holds
each time, but I could not find a proof. Is this a known formula?

\begin{python}
print("\nBinomial test formula")

import scipy.special as special

n = random.randint(1,20)
m = random.randint(0,10)

print(n,m)

print( n * special.binom(n+m, n) )

print( 1 / sum([special.binom(m,k) * (-1)**k * (1/(n+k)) 
    for k in range(m+1)]) )
\end{python}
}

\item We use Fubini (F) in the following line of equations, with indicator
  functions $\indic{S}$ for a subset $S$; it sends elements $x\in S$
  to 1 and $x\not\in S$ to $0$.
\[ \hspace*{-3em}\begin{array}[b]{rcl}
\lefteqn{\displaystyle\int_{\vec{x}\in N} d_{n-1}\big(\aneMlt(h)(\vec{\alpha})\big)
   (\vec{x}) \intd \vec{x}}
\\
& = &
\displaystyle\int_{\vec{x}\in [0,1]^{n-1}} \indic{N}(\vec{x}) \cdot
   d_{n-1}\big(\alpha_{1}+\alpha_{2}, \alpha_{3}, \ldots, \alpha_{n})\big)
   (\vec{x}) \intd \vec{x}
\\
& \smash{\stackrel{\eqref{OneSumEqn}}{=}} &
\displaystyle\int_{\vec{x}\in [0,1]^{n-1}} \indic{N}(\vec{x}) \cdot
   \int_{y\in [0,1]} \indic{(0,x_{1})}(y) \cdot
   d_{n}\big(\vec{\alpha})\big)(y, x_{1}-y, x_{2}, \ldots, x_{n}) 
   \rlap{$\intd y \intd \vec{x}$}
\\
& \smash{\stackrel{\text{(F)}}{=}} &
\displaystyle\int_{(\vec{x},y)\in [0,1]^{n}} 
   \indic{N}(\vec{x}) \cdot \indic{(0,x_{1})}(y) \cdot
   d_{n}\big(\vec{\alpha})\big)(y, x_{1}-y, x_{2}, \ldots, x_{n}) 
   \intd (\vec{x},y)
\\
& = &
\displaystyle\int_{\vec{y}\in\aneDst(h)^{-1}(N)} 
   d_{n}(\vec{\alpha})(\vec{y}) \intd \vec{y}.
\end{array} \eqno{\QEDbox} \]
\end{enumerate}
\end{proof}

\auxproof{
Once we know the basic aggregation form~\eqref{OneSumEqn} we can see
how to combine different summations of hyperparameters $\alpha_i$, as
in:
\[ \begin{array}{rcl}
\lefteqn{d_{n-2}(\alpha_{1}+\alpha_{2}+\alpha_{3}, \alpha_{3}, \ldots, \alpha_{n})
   (x_{3}, \ldots, x_{n})}
\\
& = &
\displaystyle\int_{0}^{x_{3}} \int_{0}^{x_{3}-y_{1}} 
   d_{n}(\alpha_{1}, \alpha_{2}, \alpha_{3}, \ldots, \alpha_{n})
   (y_{1}, y_{2}, x_{3} - y_{1} - y_{2}, x_{4}, \ldots, x_{n}) \intd y_{2} \intd y_{1}
\\
\lefteqn{d_{n-2}(\alpha_{1}+\alpha_{2}, \alpha_{3} + \alpha_{4}, 
   \alpha_{3}, \ldots, \alpha_{n})(x_{3}, \ldots, x_{n})}
\\
& = &
\displaystyle\int_{0}^{x_{3}} \int_{0}^{x_{4}} 
   d_{n}(\alpha_{1}, \alpha_{2}, \alpha_{3}, \ldots, \alpha_{n})
   (y_{1}, x_{3} - y_{2}, y_{2}, x_{4} - y_{2}, x_{5}, \ldots, x_{n}) \intd y_{2} \intd y_{1}
\end{array} \]

We can uniformly capture all these forms by using the funtoriality of
$\Mlt$ and $\Dst$. It gives a succint functorial reformulation of the
aggregation property of the Dirichlet functions~\cite{FrigyikKG10}.
}

Once we know the basic aggregation form~\eqref{OneSumEqn} we can
generalise by using the funtoriality of $\Mlt$ and $\Dst$. It gives a
succint functorial reformulation of the aggregation property of the
Dirichlet functions~\cite{FrigyikKG10}.

\begin{lemma}
\label{ManySumLem}
Let $h\colon n \twoheadrightarrow m$ be a surjective function.  For
$\vec{\alpha} \in \aneMlt(n)$ and $\vec{x}\in\aneDst(m)$ one has:
\begin{equation}
\label{ManySumEqn}
\begin{array}{rcl}
d_{m}\Big(\aneMlt(h)(\vec{\alpha})\Big)(\vec{x})
& = &
\displaystyle\int_{\vec{y}\in\aneDst(h)^{-1}(\vec{x})} d_{n}(\vec{\alpha})(\vec{y})
   \intd \vec{y}.
\end{array}
\end{equation}

\noindent Moreover, for a measurable subset $N\subseteq \aneDst(m)$:
\begin{equation}
\label{ManySumSubsetEqn} 
\begin{array}{rcl}
\displaystyle\int_{\vec{x}\in N} d_{m}\Big(\aneMlt(h)(\vec{\alpha})\Big)
   (\vec{x}) \intd \vec{x}
& = &
\displaystyle\int_{\vec{y}\in\aneDst(h)^{-1}(N)} 
   d_{n}(\vec{\alpha})(\vec{y}) \intd \vec{y}.
\end{array}
\end{equation}
\end{lemma}

\auxproof{
\noindent\textbf{Example.} Consider $h \colon 6 \rightarrow 3$ given by:
$$h(1)=1 \quad
h(2)=2 \quad
h(3)=1 \quad
h(4)=3 \quad
h(5)=1 \quad
h(6)=2$$

\noindent For $\vec{x} = (x_{1},x_{2},x_{3}) \in \Dst(3)$ and
$\vec{y} = (y_{1}, y_{2}, y_{3}, y_{4}, y_{5}, y_{6}) \in \Dst(6)$ we
have:
\[ \begin{array}{rcl}
\vec{y} \in \aneDst(h)^{-1}(\vec{x})
& \Longleftrightarrow &
\aneDst(h)(\vec{y}) = \vec{x}
\\
& \Longleftrightarrow &
\left\{\begin{array}{l}
x_{1} = \sum_{i\in f^{-1}(1)} y_{i} = y_{1} + y_{3} + y_{5}
\\
x_{2} = \sum_{i\in f^{-1}(2)} y_{i} = y_{2} + y_{6}
\\
x_{3} = \sum_{i\in f^{-1}(3)} y_{i} = y_{4}.
\end{array}\right.
\end{array} \]

\noindent For $\vec{\alpha} \in \aneMlt(6)$ we have:
\[ \begin{array}{rcl}
\aneMlt(h)(\vec{\alpha})
& = &
\big(\, \aneMlt(h)(\vec{\alpha})(j) \,\big)_{j\in 3}
\\
& = &
\big(\, \sum_{i\in f^{-1}(1)}\alpha_{i}, \, \sum_{i\in f^{-1}(2)}\alpha_{i}, \,
   \sum_{i\in f^{-1}(3)}\alpha_{i} \,\big)
\\
& = &
\big(\, \alpha_{1} + \alpha_{3} + \alpha_{5}, \, 
   \alpha_{2}+\alpha_{6}, \, \alpha_{4} \,\big).
\end{array} \]

\noindent Equation~\eqref{ManySumEqn} can now be written as:
\[ \begin{array}{rcl}
\lefteqn{d_{3}\big(\alpha_{1} + \alpha_{3} + \alpha_{5}, \,
   \alpha_{2}+\alpha_{6}, \, \alpha_{4} \,\big)(\vec{x})}
\\
& = &
\displaystyle \int_{0}^{x_{1}} \int_{0}^{x_{1}-z_{1}} \int_{0}^{x_{2}}
   d_{6}(\vec{\alpha})
   (z_{1}, z_{3}, z_{2}, x_{3}, x_{1} - z_{1} - z_{2}, x_{2} - z_{3})
   \intd z_{3} \intd z_{2} \intd z_{1}
\end{array} \]

\noindent Indeed we see that for the (second sequence of) arguments
of $d_6$, call them $\vec{y}$, one has:
$$y_{1} + y_{3} + y_{5} = x_{1} \qquad
y_{2} + y_{6} = x_{2} \qquad
y_{4} = x_{3}$$

\noindent so that $\vec{y} \in \aneDst(h)(\vec{x})$, as described above.
}

\auxproof{
Our next aim is to obtain a naturality result for Dirichlet, as in
Lemma~\ref{FrequentistNatLem} for MLE, for the Dirichlet distribution
functions $\Dir_{n} = \mathcal{I} \after d_{n} \colon \aneMlt(n)
\rightarrow \Giry\big(\aneDst(n)\big)$
from~\eqref{DirichletDistributionEqn}.
}

We thus arrive at the main result of this section.

\begin{lemma}
The Dirichlet maps $\Dir_n$ are natural for \emph{surjective}
functions, that is, for a surjective function $h\colon n
\twoheadrightarrow m$ the following diagram commutes.
\[ \vcenter{\xymatrix@C+2pc@R-1pc{
\aneMlt(n)\ar[r]^-{\aneMlt(h)}\ar[d]_{\Dir_{n}}
   & \aneMlt(m)\ar[d]^{\Dir_{m}}
\\
\Giry\big(\aneDst(n)\big)\ar[r]_-{\Giry(\aneDst(h))}
   & \Giry\big(\aneDst(m)\big)
}} \] 


\end{lemma}

\begin{proof}

\noindent For $\vec{\alpha} \in \aneMlt(n)$
and $N\subseteq \aneDst(m)$ we have:
\[ \begin{array}[b]{rcl}
\big(\Dir_{m} \after \aneMlt(h)\big)(\vec{\alpha})(N)
& = &
\displaystyle \int_{\vec{x}\in N} d_{m}\big(\aneMlt(h)(\vec{\alpha})\big)(\vec{x}) 
   \intd \vec{x}
\\[+0.8em]
& \smash{\stackrel{\eqref{ManySumSubsetEqn}}{=}} &
\displaystyle\int_{\vec{y}\in\aneDst(h)^{-1}(N)} d_{n}(\vec{\alpha})(\vec{y})
   \intd \vec{y}
\\
& = &
\mathcal{I}\big(d_{n}(\vec{\alpha})\big)\big(\aneDst(h)^{-1}(N)\big)
\\
& \smash{\stackrel{\eqref{GiryFunEqn}}{=}} &
\Giry\big(\aneDst(h)\big)\big(\Dir_{n}(\vec{\alpha})\big)(N)
\\
& = &
\big(\Giry\big(\aneDst(h)\big) \after \Dir_{n}\big)(\vec{\alpha})(N).
\end{array} \eqno{\QEDbox} \]
\end{proof}

\auxproof{
In Lemma~\ref{FrequentistNatLem} we have used naturality to illustrate how
taking marginals (of tables and of distributions) commutes with the
likelihood function. The corresponding result in the Bayesian setting
takes a slightly different form, because of the use of \emph{two}
functors $\Giry$ and $\aneDst$ in the composite $\Giry\aneDst$.
It now amounts to:
\[ \vcenter{\xymatrix@C+2pc{
\aneMlt(n_{1}\times n_{2})\ar[r]^-{\aneMlt(\pi_{i})}\ar[d]_{\Dir_{n_{1}\times n_{2}}}
   & \aneMlt(n_{i})\ar[d]^{\Dir_{n_i}}
\\
\Giry\big(\aneDst(n_{1}\times n_{2})\big)\ar[r]^-{\Giry(\aneDst(\pi_{i}))}
   & \Giry\big(\aneDst(n_{i})\big)
}} \] 

\noindent More explicitly, for a table $\vec{\alpha} =
\sum_{ij}\alpha_{ij}\ket{ij} \in \aneMlt(n_{1}\times n_{2})$ and for a
measurable subset $N\subseteq \aneDst(n_{1})$ we have:
\[ \begin{array}{rcl}
\displaystyle\int_{\vec{x}\in N} d_{n_1}\big(\textstyle
   \sum_{i} (\sum_{j}\alpha_{ij})\ket{i}\big)(\vec{x})\intd \vec{x}
& = &
\displaystyle\int_{\vec{y}\in \aneDst(\pi_{1})^{-1}(N)} d_{n_{1}\times n_{2}}
   (\vec{\alpha})(\vec{y})\intd \vec{y}.
\end{array} \]

\noindent This is simply a special case of
Equation~\eqref{ManySumSubsetEqn}.

\auxproof{
\[ \begin{array}{rcl}
\lefteqn{d_{n\times m}\big(\dst(\vec{\alpha}, \vec{\beta})\big)
   \big(\dst(\vec{x},\vec{y})\big)}
\\
& = &
\\
& = &
\left(\displaystyle\frac{\Gamma(\sum_{i}\alpha_{i})}{\prod_{i}\Gamma(\alpha_{i})}
   \cdot \textstyle\prod_{i} x_{i}^{\alpha_{i}-1}\right) \cdot
\left(\displaystyle\frac{\Gamma(\sum_{j}\beta_{j})}{\prod_{i}\Gamma(\beta_{j})}
   \cdot \textstyle\prod_{j} x_{j}^{\beta_{j}-1}\right) \cdot
\\
& = &
d_{n}(\vec{\alpha})(\vec{x})\cdot d_{m}(\vec{\beta})(\vec{y})
\end{array} \]

\[ \xymatrix{
\aneMlt(n)\times\aneMlt(m)\ar[rr]^-{\dst}\ar[d]_{\Dir_{n}\times\Dir_{m}} 
    & & \aneMlt(n\times m)\ar[d]^-{\Dir_{n\times m}}
\\
\Giry\big(\aneDst(n)\big)\times\Giry\big(\aneDst(m)\big)\ar[r]^-{\dst} 
   & \Giry\big(\aneDst(n)\times\aneDst(m)\big)\ar[r]^-{\Giry(\dst)}
   & \Giry\big(\aneDst(n\times m)\big)
} \]

\noindent For a measurable subset $K\subseteq\aneDst(n\times m)$,
\[ \begin{array}{rcl}
\lefteqn{\big(\Giry(\dst) \after \dst \after (\Dir_{n}\times\Dir_{m})\big)
   \textstyle(\sum_{i}\alpha_{i}\ket{i}, \sum_{j}\beta_{j}\ket{j})(K)}
\\
& = &
\dst\big(\Dir_{n}(\sum_{i}\alpha_{i}\ket{i}), 
   \Dir_{m}(\sum_{j}\beta_{j}\ket{j})\big)(\dst^{-1}(K))
\\
& = &
\\
& = &
\displaystyle\int_{\vec{z}\in K} \textstyle
   d_{n\times m}\big(\sum_{ij}\alpha_{i}\cdot\beta_{j}\ket{ij}\big)(\vec{z}) 
   \intd\vec{z}
\\
& = &
\Dir_{n\times m}\big(\sum_{ij}\alpha_{i}\cdot\beta_{j}\ket{ij}\big)(K)
\\
& = &
\big(\Dir_{n\times m} \after \dst\big)
   (\sum_{i}\alpha_{i}\ket{i}, \sum_{j}\beta_{j}\ket{j})(K).
\end{array} \]
}
}

\section{A logical perspective on learning}\label{LogicSec}

This section examines frequentist and Bayesian learning from a logical
perspective, using the notions of validity $\models$ and conditioning.
Equation~\eqref{DirichletPointEqn} expresses that the expected value
of data element $i$ under the Bayesian interpretation coincides with
the probability of $i$ under a frequentist interpretation. We shall
extend this correspondence in terms of predicates and their validity
$\models$, as explained at the end of Section~\ref{DiscPrereqSec}
and~\ref{ContPrereqSec}.

For a predicate $p\colon n \rightarrow [0,1]$ we can write $(-)\models
p \colon \Dst(n) \rightarrow [0,1]$ for the map $\omega \mapsto
\omega\models p$. It is in fact the Kleisli extension of $p$. This map
$\widehat{p} \coloneqq (-)\models p$ is now a predicate on $\Dst(n)$
from the continuous perspective. Hence we can look at $\widehat{p}$'s
validity. Below we relate validity $\models$ for the discrete
probability distribution $\ell_{n}(\vec{\alpha})$ to validity
$\models$ for the continuous probability distribution
$\Dir_{n}(\vec{\alpha})$ via an adaptation of the predicate, from $p$
to $\widehat{p}$. It provides a fancy extension
of~\eqref{DirichletPointEqn}.

\begin{proposition}
\label{DirPredProp}
For each $\vec{\alpha}\in\aneMlt(n)$ and predicate $p\colon n
\rightarrow [0,1]$ one has:
\[ \begin{array}{rcl}
\ell_{n}(\vec{\alpha}) \models p
& \;=\; &
\Dir_{n}(\vec{\alpha}) \models \widehat{p}
\end{array}
\qquad\mbox{\ie}\qquad
\vcenter{\xymatrix@C+1pc@R-1pc{
\aneMlt(n)\ar[r]^-{\ell_n}\ar[d]_{\Dir_n} & \aneDst(n)\ar[d]^-{(-)\models p}
\\
\Giry\big(\aneDst(n)\big)\ar[r]_-{(-)\models\widehat{p}} & [0,1]
}} \]

\noindent where $\widehat{p} \coloneqq (-)\models p$ is a predicate
on $\aneDst(n)$.
\end{proposition}

\begin{proof}
We unpack the definitions and use Equation~\eqref{DirichletPointEqn}
in a crucial manner:
\[ \begin{array}[b]{rcl}
\Dir_{n}(\vec{\alpha}) \models \big((-)\models p\big)
& = &
\displaystyle\int \big((-)\models p\big) \intd\, \Dir_{n}(\vec{\alpha})
\\
& = &
\displaystyle\int (\vec{x} \models p) \cdot 
   d_{n}(\vec{\alpha})(\vec{x}) \intd \vec{x}
\\
& = &
\displaystyle\int \textstyle \big(\sum_{i}\, p(i)\cdot x_{i}\big) \cdot 
   d_{n}(\vec{\alpha})(\vec{x}) \intd \vec{x}
\\
& = &
{\displaystyle\sum}_{i}\; p(i)\cdot \displaystyle\int x_{i} \cdot 
   d_{n}(\vec{\alpha})(\vec{x}) \intd \vec{x}
\\[+0.8em]
& \smash{\stackrel{\eqref{DirichletPointEqn}}{=}} &
{\displaystyle\sum}_{i}\; p(i)\cdot \displaystyle \frac{\alpha_i}{\alpha}
   \qquad \mbox{where } \alpha \coloneqq \textstyle \sum_{i}\alpha_{i}
\\[+0.8em]
& \smash{\stackrel{\eqref{FrequentistEqn}}{=}} & 
{\displaystyle\sum}_{i}\; p(i)\cdot \ell_{n}(\vec{\alpha})(i)
\hspace*{\arraycolsep}=\hspace*{\arraycolsep}
\ell_{n}(\vec{\alpha}) \models p.
\end{array} \eqno{\QEDbox} \]
\end{proof}

The Bayesian approach to learning can handle additional data via
conditioning. This will be made precise below, using a logical
formulation that is characteristic for conjugate priors,
see~\cite{Jacobs17d}. For numbers $n$ and $i\in n$ we write
$\indic{\{i\}}\colon n \rightarrow [0,1]$ for the singleton predicate
that is $1$ on $i\in n$ and $0$ elsewhere.

\begin{theorem}
\label{DirConditionProp}
For $\vec{\alpha} = (\alpha_{1}, \ldots, \alpha_{n})\in\aneMlt(n)$ and
$i\in n$ write $\incr{\vec{\alpha}}{i}$ for the sequence $(\alpha_{1},
\ldots, \alpha_{i}+1,\ldots, \alpha_{n})$ in which the $i$-th entry is
is incremented by one. A Dirichlet distribution updated with a point
$i$ observation is the same as this distribution with the $i$-th
hyperparameter incremented by one:
\[ \begin{array}{rcl}
\Dir_{n}(\incr{\vec{\alpha}}{i})
& = &
\Dir_{n}(\vec{\alpha})\big|_{\widehat{\indic{\{i\}}}}.
\end{array} \]
\end{theorem}

\begin{proof}
Write $\alpha = \sum_{i}\alpha_{i}$. By Proposition~\ref{DirPredProp} we have:
\[ \begin{array}{rcccccl}
\Dir_{n}(\vec{\alpha}) \models \widehat{\indic{\{i\}}}
& \;=\; &
\ell_{n}(\vec{\alpha}) \models \indic{\{i\}}
& \;=\; &
\ell_{n}(\vec{\alpha})(i)
& \;=\; &
\displaystyle\frac{\alpha_i}{\alpha}.
\end{array} \]

\noindent Thus, for a measurable subset $M\subseteq\aneDst(n)$,
\[ \begin{array}[b]{rcl}
\Dir_{n}(\vec{\alpha})\big|_{\widehat{\indic{\{i\}}}}(M)
& \smash{\stackrel{\eqref{ContConditionEqn}}{=}} &
\displaystyle \frac{\int_{M} \widehat{\indic{\{i\}}} \intd\, 
   \Dir_{n}(\vec{\alpha})}
   {\Dir_{n}(\vec{\alpha}) \models \widehat{\indic{\{i\}}}}
\\[+1em]
& = &
\displaystyle \frac{\alpha}{\alpha_i} \cdot 
   \int_{\vec{x}\in M} \big(\vec{x}\models\indic{\{i\}}\big) \cdot
   d_{n}(\vec{\alpha})(\vec{x}) \intd \vec{x}
\\[+1em]
& = &
\displaystyle \frac{\alpha}{\alpha_i} \cdot \int_{\vec{x}\in M} x_{i} \cdot
   \frac{\Gamma(\sum_{j}\alpha_{j})}{\prod_{j} \Gamma(\alpha_{j})}
    \cdot {\textstyle\prod_{j}} x_{j}^{\alpha_{j}-1} \intd \vec{x}
\\[+1em]
& = &
\displaystyle \int_{\vec{x}\in M}
   \frac{\Gamma(\sum_{j}\alpha_{j} + 1)}
    {\Gamma(\alpha_{i}+1)\cdot \prod_{j\neq i} \Gamma(\alpha_{j})}
    \cdot x_{i}^{\alpha_i}\cdot {\textstyle\prod_{j\neq i}} x_{j}^{\alpha_{j}-1} 
    \intd \vec{x}
\\[+1em]
& = &
\displaystyle \int_{\vec{x}\in M} d_{n}(\incr{\vec{\alpha}}{i})(\vec{x})
   \intd \vec{x}
\hspace*{\arraycolsep}=\hspace*{\arraycolsep}
\Dir_{n}(\incr{\vec{\alpha}}{i})(M).
\end{array} \eqno{\QEDbox} \]
\end{proof}


This update property does not work for the frequentist approach:
update of an existing distribution with point evidence trivialises the
distribution:
\[ \begin{array}{rcccccl}
\ell_{n}(\vec{\alpha})\big|_{\indic{\{i\}}}
& = &
\displaystyle\sum_{j} \frac{\ell_{n}(\vec{\alpha})(j) \cdot \indic{\{i\}}(j)}
   {\ell_{n}(\vec{\alpha}) \models \indic{\{i\}}}\bigket{j}
& = &
\displaystyle\frac{\ell_{n}(\vec{\alpha})(i)}{\ell_{n}(\vec{\alpha})(i)}
   \bigket{i}
& = &
1\ket{i}.
\end{array} \]

\section{A local Bayesian approach}\label{LocalBayesianSec}

The question that we wish to analyse in this final section is: suppose
we have multidimensional data, together with a graph structure for a
Bayesian network. Can we do Bayesian learning of the parameters for
this network also separately (locally)? Since there are so many
parameters involved, we simplify the situation to a $2\times 3$
example and proceed in a less formal manner, focusing on intuitions.

We look at a 2-dimensional table, of size $2\times 3$, given by a
multiset $\phi = \alpha_{11}\ket{11} + \alpha_{12}\ket{12} +
\alpha_{13}\ket{13} + \alpha_{21}\ket{21} + \alpha_{22}\ket{22} +
\alpha_{23}\ket{23}$. We think about this situation in terms of an
initial state $\omega\in\Dst(2)$ and a channel $c\colon 2 \rightarrow
\Dst(3)$. We like to learn the (parameters) of the three distributions
$\omega\in\Dst(2)$, $c(1) \in \Dst(3)$, $c(2) \in \Dst(3)$ separately
--- in a Bayesian manner, via Dirichlet.  The question we wish to
address is how this is related to learning a joint distribution in
$\Dst(2\times 3)$, via the parameters $\alpha_{ij}$.  Here we use that
disintegration gives translations back and forth beween
$\aneDst(2)\times\aneDst(3)\times\aneDst(3)$ and $\aneDst(2\times 3)$.
For this we use the following function.
\[ \xymatrix@C+1pc{
\aneDst(6)\ar[rr]^-{h = \tuple{h_{1},h_{2},h_{3}}} & &
   \aneDst(2)\times\aneDst(3)\times\aneDst(3)
} \]

\noindent where:
\[ \begin{array}{rcccl}
h_{1}(\vec{x})
& \coloneqq &
(\sum_{j}x_{1j}, \sum_{j}x_{2j})
& = &
(y_{1},y_{2})
\end{array}
\qquad\mbox{and}\qquad
\left\{\begin{array}{rcl}
h_{2}(\vec{x})
& \coloneqq &
(\frac{x_{11}}{y_1}, \frac{x_{12}}{y_1}, \frac{x_{13}}{y_1})
\\
h_{3}(\vec{x})
& \coloneqq &
(\frac{x_{21}}{y_2}, \frac{x_{22}}{y_2}, \frac{x_{23}}{y_2})
\end{array}\right. \]

We start by the following equations; their proofs are easy and left to
the reader.  For the above hypeparameters $\vec{\alpha}$ we write
$\beta_{1} = \alpha_{11} + \alpha_{12} + \alpha_{13}$ and $\beta_{2} =
\alpha_{21} + \alpha_{22} + \alpha_{23}$. Then:
\[ \hspace*{-0.5em}\begin{array}{rcl}
\lefteqn{\; d_{6}(\vec{\alpha})(\vec{x})}
\\
& = &
\displaystyle
   \frac{d_{2}(\beta_{1}, \beta_{2})\big(h_{1}(\vec{x})\big)}
   {y_{1}^{2}\cdot y_{2}^{2}}
\cdot \textstyle d_{3}(\alpha_{11},\alpha_{12},\alpha_{13})\big(h_{2}(\vec{x})\big)
\cdot d_{3}(\alpha_{21},\alpha_{22},\alpha_{23})\big(h_{3}(\vec{x})\big)
\\[+1em]
& = &
\displaystyle\frac{(\beta_{1}\!+\!\beta_{2}-1)
   (\beta_{1}\!+\!\beta_{2}-2)
   (\beta_{1}\!+\!\beta_{2}-3)
   (\beta_{1}\!+\!\beta_{2}-4)}
   {(\beta_{1}-1)(\beta_{1}-2)(\beta_{2}-1)(\beta_{2}-2)} \;\cdot
\\
& & \quad
d_{2}(\beta_{1}-2, \beta_{2}-2)\big(h_{1}(\vec{x})\big)
\cdot \textstyle d_{3}(\alpha_{11},\alpha_{12},\alpha_{13})\big(h_{2}(\vec{x})\big)
\cdot d_{3}(\alpha_{21},\alpha_{22},\alpha_{23})\big(h_{3}(\vec{x})\big)
\end{array} \]

\auxproof{
\[ \begin{array}{rcl}
\lefteqn{\displaystyle \frac{d_{2}(\vec{\beta})}{y_{1}^{2}\cdot y_{2}^{2}}
   \cdot d_{3}(\vec{\alpha_{1-}})(\frac{\vec{x_{1-}}}{y_{1}})
   \cdot d_{3}(\vec{\alpha_{2-}})(\frac{\vec{x_{2-}}}{y_{2}})}
\\[+1em]
& = &
\displaystyle\frac{\Gamma(\beta_{1}+\beta_{2})}
   {\Gamma(\beta_{1})\cdot\Gamma(\beta_{2})}
   \textstyle y_{1}^{\beta_{1}-3}y_{2}^{\beta_{2}-3} \cdot
\displaystyle\frac{\Gamma(\beta_{1})}{\prod_{j}\Gamma(\alpha_{1j})}
   \textstyle \prod_{j} (\frac{x_{1j}}{y_{1}})^{\alpha_{1j}-1} \cdot
\displaystyle\frac{\Gamma(\beta_{2})}{\prod_{j}\Gamma(\alpha_{2j})}
   \textstyle \prod_{j} (\frac{x_{2j}}{y_{2}})^{\alpha_{2j}-1}
\\[+1em]
& = &
\displaystyle\frac{\Gamma(\sum_{ij}\alpha_{ij})}{\prod_{ij}\Gamma(\alpha_{ij})}
   \textstyle \prod_{ij}x_{ij}^{\alpha_{ij}-1}
\\
& = &
d_{6}(\vec{\alpha})(\vec{x})
\end{array} \]

\[ \begin{array}{rcl}
\displaystyle \frac{d_{2}(\vec{\beta})(\vec{y})}{y_{1}^{2}\cdot y_{2}^{2}}
& = &
\displaystyle\frac{\Gamma(\beta_{1}+\beta_{2})}
   {\Gamma(\beta_{1})\cdot\Gamma(\beta_{2})}
   \textstyle y_{1}^{\beta_{1}-3}y_{2}^{\beta_{2}-3}
\\[+1em]
& = &
\displaystyle\frac{(\beta_{1}+\beta_{2}-1)(\beta_{1}+\beta_{2}-2)
   (\beta_{1}+\beta_{2}-3)(\beta_{1}+\beta_{2}-4)}
   {(\beta_{1}-1)(\beta_{1}-2)(\beta_{2}-1)(\beta_{2}-2)}
\\[+1em]
& & \qquad \cdot \; \displaystyle
\frac{\Gamma((\beta_{1}-2)+(\beta_{2}-2))}
   {\Gamma(\beta_{1}-2)\cdot\Gamma(\beta_{2}-2)}
   \textstyle y_{1}^{(\beta_{1}-2)-1}y_{2}^{(\beta_{2}-2)-1}
\\[+1em]
& = &
\displaystyle\frac{(\beta_{1}+\beta_{2}-1)(\beta_{1}+\beta_{2}-2)
   (\beta_{1}+\beta_{2}-3)(\beta_{1}+\beta_{2}-4)}
   {(\beta_{1}-1)(\beta_{1}-2)(\beta_{2}-1)(\beta_{2}-2)} \cdot
   d_{2}(\vec{\beta}-2)(\vec{y}).
\end{array} \]
}

\noindent The main result of this section now relates updates of
``joint'' hyperparameters to updates of ``local'' hyperparameters of
the associated channel --- via a particular instantiation. This gives
the essence of how to do parameter learning for a Bayesian network,
see~\cite[\S17.4]{KollerF09}. It builds on
Theorem~\ref{DirConditionProp}, and uses the notation $\incr{}{}$
introduced there.

\begin{theorem}
\label{JointLocalUpdateThm}
In the situation described above,
\[ \begin{array}{rcl}
\lefteqn{\Giry(h)\Big(\Dir_{6}(\incr{\vec{\alpha}}{(1,3)})\Big)}
\\
& = &
C \cdot \Dir_{2}(\incr{(\vec{\beta}-2)}{1}) \otimes
   \Dir_{3}(\incr{\vec{\alpha_{1-}}}{3}) \otimes 
   \Dir_{3}(\vec{\alpha_{2-}}),
\end{array} \]

\noindent where $\beta_{i} = \sum_{j}\alpha_{ij}$ and the constant $C$ is:
\[ \begin{array}{rcl}
C
& = &
\displaystyle\frac{(\beta_{1}\!+\!\beta_{2})(\beta_{1}\!+\!\beta_{2}-1)
   (\beta_{1}\!+\!\beta_{2}-2)
   (\beta_{1}\!+\!\beta_{2}-3)}
   {\beta_{1}(\beta_{1}-1)(\beta_{2}-1)(\beta_{2}-2)}
\end{array} \eqno{\QEDbox} \]
\end{theorem}

\auxproof{
\begin{proof}
We use the relation established in Theorem~\ref{DirConditionProp}
between incrementing hyperparameters (via $\incr{}{}$) and updating
with singleton predicates, $\indic{\{x\}}$ translated via
$\widehat{\;}$ as in Proposition~\ref{DirPredProp}. Below we shall
write $C,C'$ for suitable constants that will be determined later. For
measurable subsets $K\subseteq \aneDst(2)$ and $M,N\subseteq
\aneDst(3)$ we have:
\[ \begin{array}{rcl}
\lefteqn{\Giry(h)\Big(\Dir_{6}(\incr{\vec{\alpha}}{(1,3)})\Big)
   (K\times M\times N)}
\\
& = &
\Dir_{6}(\vec{\alpha})\big|_{\widehat{\indic{\{x_{13}\}}}}
   \big(h^{-1}(K\times M\times N)\big)
\\
& = &
\displaystyle\frac{\int_{\vec{x}\in h^{-1}(K\times M\times N)} x_{13} \cdot
   d_{6}(\vec{\alpha})(\vec{x}) \intd \vec{x}}
  {\Dir_{6}(\vec{\alpha})\models \widehat{\indic{\{x_{13}\}}}}
\\[+1em]
& = &
C\cdot \displaystyle \frac{\int_{\vec{y}\in K, \vec{u} \in M, \vec{v} \in N}
   y_{1}\cdot u_{3} \cdot d_{2}(\vec{\beta}-2)(\vec{y})
   \cdot d_{3}(\vec{\alpha_{1-}})(\vec{u})
   \cdot d_{3}(\vec{\alpha_{2-}})(\vec{v}) \intd \vec{y}\,\vec{u}\,\vec{v} }
   {\nicefrac{\alpha_{13}}{(\beta_{1}+\beta_{2})}}
\\[+1em]
& = &
C\cdot \displaystyle \frac{\int_{\vec{y}\in K} 
   y_{1} \cdot d_{2}(\vec{\beta}-2)(\vec{y}) \intd \vec{y} \cdot 
   \int_{\vec{u} \in M} u_{3} \cdot d_{3}(\vec{\alpha_{1-}})(\vec{u}) \intd \vec{u} 
   \cdot \int_{\vec{v} \in N} d_{3}(\vec{\alpha_{2-}})(\vec{v}) \intd \vec{v}}
   {\nicefrac{\beta_{1}}{(\beta_{1}-2)} \cdot
    \nicefrac{(\beta_{1}+\beta_{2}-4)}{\beta_{1}+\beta_{2}} \cdot
    \nicefrac{(\beta_{1}-2)}{(\beta_{1}+\beta_{2}-4)} \cdot
    \nicefrac{\alpha_{13}}{\beta_{1}}}
\\[+1em]
& = &
C'\cdot \displaystyle \frac{\int_{\vec{y}\in K} 
   y_{1} \cdot d_{2}(\vec{\beta}-2)(\vec{y})\intd \vec{y}}
   {\nicefrac{(\beta_{1}-2)}{(\beta_{1}+\beta_{2}-4)}}  \cdot 
   \frac{\int_{\vec{u} \in M} u_{3} \cdot d_{3}(\vec{\alpha_{1-}})(\vec{u}) 
   \intd \vec{u}}{\nicefrac{\alpha_{13}}{\beta_{1}}}
   \cdot \Dir_{3}(\vec{\alpha_{2-}})(N)
\\[+1em]
& = &
C'\cdot \displaystyle \frac{\int_{\vec{y}\in K} 
   y_{1} \cdot d_{2}(\vec{\beta}-2)(\vec{y})\intd \vec{y}}
   {\Dir_{2}(\vec{\beta}-2) \models \widehat{\indic{\{y_1\}}}}  \cdot 
   \frac{\int_{\vec{u} \in M} u_{3} \cdot d_{3}(\vec{\alpha_{1-}})(\vec{u}) 
   \intd \vec{u}}{\Dir_{2}(\vec{\alpha_{1-}}) \models \widehat{\indic{\{u_3\}}}}
   \cdot \Dir_{3}(\vec{\alpha_{2-}})(N)
\\[+1em]
& = &
C'\cdot \Dir_{2}(\vec{\beta}-2)\big|_{\widehat{\indic{\{y_1\}}}}(K) \cdot
   \Dir_{2}(\vec{\alpha_{1-}})\big|_{\widehat{\indic{\{u_3\}}}}(M) \cdot
   \Dir_{3}(\vec{\alpha_{2-}})(N)
\\
& = &
C'\cdot \Big(\Dir_{2}(\incr{(\vec{\beta}-2)}{1}) \otimes
   \Dir_{3}(\incr{\vec{\alpha_{1-}}}{3}) \otimes 
   \Dir_{3}(\vec{\alpha_{2-}})\Big)(K\times M\times N).
\end{array} \]

\noindent The constant involved is:
\[ \hspace*{-0.5em}\begin{array}[b]{rcl}
C'
& = &
\displaystyle C \cdot \frac{1}{\nicefrac{\beta_{1}}{(\beta_{1}-2)} \cdot
    \nicefrac{(\beta_{1}\!+\!\beta_{2}-4)}{(\beta_{1}\!+\!\beta_{2})}}
\\[+1em]
& = &
\displaystyle\frac{(\beta_{1}\!+\!\beta_{2}-1)
   (\beta_{1}\!+\!\beta_{2}-2)
   (\beta_{1}\!+\!\beta_{2}-3)
   (\beta_{1}\!+\!\beta_{2}-4)}
   {(\beta_{1}-1)(\beta_{1}-2)(\beta_{2}-1)(\beta_{2}-2)} \cdot
   \frac{(\beta_{1}-2)(\beta_{1}\!+\!\beta_{2})}
      {\beta_{1}(\beta_{1}\!+\!\beta_{2}-4)}
\\[+1em]
& = &
\displaystyle\frac{(\beta_{1}\!+\!\beta_{2})(\beta_{1}\!+\!\beta_{2}-1)
   (\beta_{1}\!+\!\beta_{2}-2)
   (\beta_{1}\!+\!\beta_{2}-3)}
   {\beta_{1}(\beta_{1}-1)(\beta_{2}-1)(\beta_{2}-2)}
\end{array} \eqno{\QEDbox} \]
\end{proof}
}

\section{Conclusions}

Specific natural transformations $\neMlt\Rightarrow\Dst$ and $\aneMlt
\Rightarrow \Giry\Dst$ have been identified as the crucial ways of
going from data to distributions in parameter learning. This
categorical approach may shed light on non-trivial applications of
learning, for instance in topic modelling via latent Dirichlet
allocation~\cite{BleiNJ03,SteyversG07}.


\end{document}